 \documentclass[final,5p,times,twocolumn]{elsarticle}

\newtheorem{theorem}{Theorem}[section]
\newtheorem{proof}{Proof}[section]

\usepackage{graphicx}
\usepackage{amsmath}
\usepackage{amssymb}
\usepackage{multicol}
\usepackage{enumerate}

\newcommand{\real}{\mathbb{R}}

\newcommand{\calS}{{\cal S}}
\newcommand{\bc}{{\bf c}}
\newcommand{\bx}{{\bf x}}

\newtheorem{remark}{Remark}
\newcommand{\s}{s(\sigma)}
\newcommand{\bas}{s(\bar\sigma)}
\newcommand{\lns}{\ln\left(\s\right)}
\newcommand{\lnbas}{\ln\left(\bas\right)}
\newcommand{\dis}{-d(c)}
\newcommand{\disb}{-d(\bar{c})}

\journal{Neurocomputing}

\begin{document}


\begin{frontmatter}

\title {CANONICAL DUAL SOLUTIONS TO NONCONVEX   RADIAL BASIS NEURAL NETWORK OPTIMIZATION PROBLEM\footnote{ This research is supported by US Air Force Office of Scientific Research under the grant AFOSR FA9550-10-1-0487.}}
\author{Vittorio Latorre$^{*}$ and David Yang Gao$^{+}$
\footnote{\em{Email addresses}: $^*$latorre@dis.uniroma1.it, $^+$d.gao@ballarat.edu.au}}

\address{$^{*}$University ``Sapienza" of Rome, Rome, Italy\\
$^{+}$University of Ballarat and Australian National University,  Australia\\}

\begin{abstract}
Radial Basis Functions Neural Networks (RBFNNs) are tools widely used in regression problems. One of their principal drawbacks is that the formulation corresponding to the training with the supervision of both the centers and the weights is a highly non-convex optimization problem, which leads to some fundamentally difficulties for 
 traditional optimization theory and   methods.
  This paper presents a generalized canonical duality theory  for solving this
 challenging   problem.   We demonstrate that by  sequential canonical dual transformations, the  
 nonconvex optimization problem of the RBFNN  can be reformulated as a canonical dual problem (without duality gap).
Both global optimal solution and local extrema can be classified. 
Several applications to one of the most used Radial Basis Functions, the Gaussian function, are illustrated.
 Our results show that even for one-dimensional case, the global minimizer of the nonconvex problem 
  may not be the best solution to the RBFNNs, and the  canonical dual theory is a promising tool
   for solving general    neural networks training problems.
\end{abstract}

\end{frontmatter}

\section{Introduction}
Radial Basis Function Neural Networks(RBFNN) are a tool introduced in the field of function interpolation \cite{powell} and then were adapted to the problem of regression \cite{hay}. During the last two decades RBFNN were applied in several fields. The problem of regression consists in trying to approximate a function $f: \real^n\rightarrow \real$ by means of an approximation function $g(\cdot)$ that uses a set of samples defined as:
\begin{equation}\label{setT}
{\cal T}=\{(x^p,y^p),x^p\in\real^n,y^p\in\real, p=1,...,P\},
\end{equation}
where $(x^p, y^p)$ are respectively arguments and values of the given function $f(x)$. In general the approximating function $g(\cdot)$ obtained by the RBFNNs with radial basis function $\phi(\cdot)$ has the following form:
\begin{equation}
g(x)=\sum_{i=1}^{N} w_i\phi(\| \bx- \bc_i\|),\label{gx}
\end{equation}
where $N$ is the number of units used to approximate the function, or neurons of the network, $\textbf{w}$ is the vector with components  $w_i$ for $i=1,\dots,N$ that is the vector of the weights associated with the connections between the units $\bx$  and $\bc_i\in\real^n$ for $i=1,\dots,N$ are the centers of the RBFNNs.

Generally speaking, there are two main optimization strategies to train a RBFNN.
The first consists in the optimization of only the weights of the neural network. In this case the centers are generally chosen by using clustering strategies \cite{bru}. This problem is a convex problem in the variable $\textbf{w}$ and has the form:
\begin{equation}\label{errw}
E(\textbf{w})=\frac{1}{2}\sum_{p=1}^{P}\sum_{i=1}^{N}(w_i\phi( \bc_i)-y_p)^2+ \frac{1}{2} \beta_w\|\textbf{w}\|^2,
\end{equation}
where $\beta_w$ is the  regularization parameter for the weights.

The second strategy is to consider both weighter  $\textbf{w}$ and the centers $\bc$ of the radial basis functions as variables. This strategy can be performed by solving the following unconstrained optimization problem:
\begin{eqnarray}\label{errc}
E(\textbf{w},\bc)&=&\frac{1}{2}\sum_{p=1}^{P}\sum_{i=1}^{N}(w_i\phi( \bc_i)-y^p)^2+\nonumber\\
&&\frac{1}{2} \beta_w\|\textbf{w}\|^2 + \frac{1}{2} \beta \sum_{i=1}^N \sum_{j=1}^n c_{ji}^2 .
\end{eqnarray}
This problem is non-convex, but from empirical experiments \cite{wett} it emerged that it generally yields neural networks with an higher precision than the ones trained with strategy (\ref{errw}). One of the most used strategies to solve this optimization problem is to apply decomposition algorithms \cite{griscia}.
However, due to the non-convexity of  the  problem (\ref{errc}), there are some fundamental difficulties to
find the global minimum of the problem and to characterize  local minima.
Indeed, the  problem (\ref{errc}) is considered to be NP-hard even if the radial basis
function $\phi(\bc)$ is a quadratic function and $n=1$ \cite{more-wu,saxe}.
Another issue that characterizes this problem is the choice of the regularization parameters $\beta_w$ and $\beta$.
 In general a cross-validation strategy is applied in order to find these regularization parameters.
 Cross-validation consists in trying different values of the parameters in order to find the one that yields the neural network with the best prediction.
 Until now it was not possible to find a closed form for the optimal values of these parameters in the general case.
 If it is possible to find at least an upper bound for these parameters, the time needed to perform a cross validation would greatly decrease.

 Canonical duality theory developed from nonconvex analysis and global optimization  \cite{gaob,gao-jogo00}
  is a potentially powerful methodology, which
has been used successfully for solving a large class of challenging problems in  biology,  engineering, sciences \cite{gao-cace09,wang-etal,zgy}, and recently in network communications \cite{g-r-p,ruan-gao-ep}.
 In this paper we study  the canonical duality theory for solving  the general
 Radial Basis Neural Networks optimization problem (\ref{errc})  and
 mainly analyze  one-dimensional case in order to find properties and intuitions that can be useful for the multidimensional cases.
 The rest of this paper is arranged as follows.
In Section 2, we first demonstrate how to  rewrite the nonconvex 
 primal problem as a dual problem by using sequential canonical dual transformation
 developed in \cite{gaob,gaot}. In Section 3 we prove the complementarity-dual principle showing that the obtained formulation is canonically dual to the original problem in the sense that there is no duality gap. In Section 4,  we analyze the problem with the Gaussian function as radial basis in the neurons and show some examples. 
  The last section  presents some conclusions.

\section{Primal problem for general Radial Basis Functions(RBF)}

 The general one dimensional non-convex function to be addressed  in this paper can be proposed in the following form:
\begin{equation}
P(c)=W(c)+ \frac{1}{2} \beta c^2-f c,
\end{equation}
where $\beta$ is the regularization coefficient and $f$ is a positive scalar close to zero. The term $-f c$ is not comprised in the original Radial Basis Neural Networks formulation but we consider it for the general mathematical case.
The non-convex function $W(c)$  depends on the choice of the radial basis function $\phi(\cdot)$:
\begin{equation}\label{W1}
W(c)=\frac{1}{2}\left(w\phi(\|x-c\|^2)-y\right)^2,
\end{equation}
where $x$, $y$  and $w$ belong to $\real$.
In   applications the parameter $w$ is also a
variable, but the original problem  (\ref{errc})  is convex in $w$ while non-convex in respect to the
center of the radial basis function $c$.
Therefore,  the one-dimensional non-convex primal problem can be formulated as
\begin{eqnarray}\label{primal}
({\cal P}): \;\; \min \Big\{ P(c)=&\frac{1}{2}\left(w\phi(\|x-c\|^2)-y\right)^2\nonumber
\\&+\frac{1}{2}\beta c^2-fc\quad \; | \;  \forall c\in \real \Big\}.
\end{eqnarray}

In order to apply the canonical duality theory to solve this problem,
 we need to choose the following geometrically nonlinear operator:
\begin{equation}\label{operator}
\xi=\Lambda(c)=w\phi(\|x-c\|^2): \quad \real \rightarrow \cal{E}\rm_a.
\end{equation}
Clearly, this  is a nonlinear  map from $\real$ to a subspace  $\cal{E}\rm_a \in \real$,
 which depends on the choice of the Radial Basis Function $\phi(\cdot)$.
The {\em canonical function}  associated with this geometrical operator  is
\begin{equation}\label{fcano}
V(\xi(c))=\frac{1}{2}(\xi(c) - y)^2 =W(\Lambda(c)).
\end{equation}
By the definition introduced in the canonical duality theory \cite{gao-jogo00},
   $V: \cal{E}\rm_a \rightarrow \real$ is said to be canonical
   function on $\cal{E}\rm_a$ if  for any given $\xi \in \cal{E}\rm_a$, the duality relation
\begin{equation}\label{dmap}
\sigma= V'(\xi)=\{\xi-y\} : \cal{E}\rm_a \rightarrow \cal{S}\rm_a
\end{equation}
is  invertible, where $\cal{S}\rm_a$ is
the range of the duality mapping $\sigma =\partial V(\xi)$,
 which depends on the choice of the Radial Basis Function $\phi(\cdot)$.
  The couple $(\xi, \sigma)$ forms a canonical duality pair on $\cal{E}\rm_a \times \cal{S}\rm_a$ with the Legendre conjugate $V^*(\sigma)$ defined by
\begin{equation}\label{conj}
V^*(\sigma)= \{\xi\sigma-V(\xi)|\sigma=V'(\xi)\}=\left(\frac{1}{2}\sigma^2+y\sigma\right).
\end{equation}
By considering that $W(c)= \Lambda(c)\sigma -V^*(\sigma)$,  the primal function $P(c)$ can be reformulated
as the so-called {\em total complementarity function} defined by
\begin{eqnarray}
\Xi(c,\sigma)&=& \Lambda(w,c)\sigma-V^*(\sigma)+\frac{1}{2}\beta c^2-fc \nonumber \\
    &=&w\phi(\|x-c\|^2)\sigma-\left(\frac{1}{2}\sigma^2 + \sigma y\right)\nonumber\\
    &&+\frac{1}{2}\beta c^2-fc\label{comp}.
\end{eqnarray}
The function $\phi(\cdot)$ can be a non convex function just like $W(c)$. For this reason we have to perform a sequential canonical dual transformation for the nonlinear operator $\Lambda(c)$. To this aim we choose a second nonlinear operator:
\begin{equation}\label{operator2}
\epsilon = \Lambda_2(c)=\|x-c\|^2
\end{equation}
which is a map from $\real$ to $\cal{E} \rm_b =\{ \epsilon\in \real | \epsilon \ge 0\}$.
In terms of $\epsilon$, the first level operator $\xi=\Lambda(c)$ can be written as
\begin{equation}
\xi= U(\epsilon)= w\phi(\epsilon ).
\end{equation}
We assume that $U(\epsilon)$ is a convex function on $\cal{E}\rm_b$ such that  the second-level duality relation
\begin{equation}\label{dmap2}
\tau= U'(\epsilon)= w \phi'(\epsilon)
\end{equation}
 is  invertible, i.e.,
\begin{equation}\label{eptau}
\epsilon=\left(\phi'\left( \frac{\tau}{w}\right )\right)^{-1},
\end{equation}
where the term $\left(\phi'\left( \frac{\tau}{w}\right )\right)^{-1}$ is the inverse of the function $\phi'(\epsilon)$.
 Thus, the Legendre conjugate of $U$ can be obtained uniquely by
\begin{equation}\label{conj2}
U^*(\tau)= \tau\left(\phi'\left( \frac{\tau}{w}\right )\right)^{-1} -
w\phi \left( \left( \phi' \left(  \frac{\tau}{w} \right)\right)^{-1}\right).
\end{equation}
We notice that $
\xi=w\phi(\epsilon)$.
By substituting the value of $\epsilon$ given by (\ref{eptau}) we find a relation that connects the first level primal variable $\xi$ with the second level dual variable $\tau$:
\begin{equation}\label{relation1}
\xi=w\phi\left(\left(\phi'\left( \frac{\tau}{w}\right )\right)^{-1}\right).
\end{equation}
By plugging this in (\ref{dmap}) we obtain
\begin{equation}
\sigma=w\phi\left(\left(\phi'\left( \frac{\tau}{w}\right )\right)^{-1}\right) -y .
\end{equation}
Generally speaking, it is possible, for certain functions $\phi$, to use the canonical dual transformation to find the relation between the first level dual variable $\sigma$ and the second level dual variable $\tau$ by means of the derivatives of $\phi(\cdot)$ and the first primal variable $\xi$. In general this relation is:
\begin{equation}\label{tau}
{\tau}= w\phi'\left(\phi^{-1}\left(\frac{{\sigma}+y}{w}\right)\right).
\end{equation}

\noindent
Therefore, replacing $U(\xi)=\Lambda(c)$ by its Legendre conjugate $U^*$, the total complementarity function becomes
\begin{eqnarray}\label{comp}
\Xi(c,\sigma,\tau)&=& \left(\|x_p-c_i\|^2\tau-U^*(\tau)\right)\sigma \nonumber\\&&-V^*(\sigma)+\frac{1}{2}\beta c^2-fc .
\end{eqnarray}
It is also possible to rewrite the total complementary function (\ref{comp}) in the following form:
\begin{eqnarray}\label{gfcomp}
\Xi(c,\sigma,\tau)&=&  \frac{1}{2}c^2 (2\tau \sigma+ \beta)
-c (2\tau \sigma x+f) \nonumber\\
&&-U^*(\tau)\sigma-V^*(\sigma)+x^2\tau \sigma.
\end{eqnarray}
By the criticality condition $\partial \Xi(c, \sigma, \tau)/\partial c = 0$  we obtain
\begin{equation}\label{devccomp}
 c (\tau,\sigma)   =\frac{2\tau x \sigma+f}{2\tau \sigma+\beta}.
\end{equation}
Clearly, if $ 2\tau \sigma+\beta \ne0$, the general solution of (\ref{devccomp}) is
\begin{equation}\label{c}
c=\frac{2{\tau} x {\sigma}+f}{2{\tau} {\sigma}+\beta} \quad \forall (\sigma,\tau)
\in \cal{S}\rm_a=\{\sigma,\tau |  \; 2\tau \sigma+\beta \ne0\}
\end{equation}
and the   canonical dual function of $P(c)$ can be presented as
\begin{equation}\label{dual}
P^d(\sigma,\tau)= - \frac{1}{2}\frac{(2\tau x \sigma+f)^2}{2\tau \sigma+\beta}  -U^*(\tau)\sigma-V^*(\sigma)+ x^2\tau \sigma.
\end{equation}
By considering dual relation given in  (\ref{tau}), and by setting $\s=\frac{{\sigma}+y}{w}$, we can write the total complementarity function in terms of only $c$ and $\sigma$
\begin{eqnarray}\label{compcs}
\Xi(c,\sigma)=& \frac{1}{2}c^2G( \sigma)-cF( \sigma)-U^*(\sigma)\sigma- \nonumber\\
&V^*(\sigma)+x^2w\phi'\left(\phi^{-1}\left(\s\right)\right) \sigma,
\end{eqnarray}
where
\begin{eqnarray*}
 G(\sigma)&=& 2w\phi'\left(\phi^{-1}\left(\s\right)\right) \sigma+\beta,\\
 F(\sigma) &=& 2w\phi'\left(\phi^{-1}\left(\s\right)\right) x \sigma+f, \\
 U^*(\sigma)&=&  w\phi'\left(\phi^{-1}\left(\s\right)\right)\phi^{-1}\left(\s\right)-(\sigma+y).
\end{eqnarray*}
Therefore, in terms of $\sigma$ only, the canonical dual function can be written as
\begin{eqnarray}\label{duals}
P^d(\sigma)&=& - \frac{1}{2}\frac{  F(\sigma)^2}{G(\sigma)}  -U^*(\sigma)\sigma +V^*(\sigma)\nonumber-\\&& x^2w\phi'\left(\phi^{-1}\left(\s\right)\right) \sigma .
\end{eqnarray}

\section{Complementary-Dual Principle}

\begin{theorem}\label{criteq}
If $\bar{\sigma}$ is a critical point of ($P^d$) and the term:
\begin{eqnarray}\label{taudev}
G'(\bar\sigma)&=&\sigma  \phi''\left(\phi^{-1}\left(\bas\right)\right)\left(\phi^{-1}\left(\bas\right)\right)'+\nonumber\\
&&w\phi'\left(\phi^{-1}\left(\bas\right)\right)\ne0 , 
\end{eqnarray}
then the point
\begin{equation}
\bar{c}=\frac{F(\bar\sigma)}{G(\bar\sigma)}
\end{equation}
is a critical point of $P(c)$ and $P(\bar{c})=P^d(\bar{\sigma})$
\end{theorem}
\begin{proof}{\em 
Suppose that $\bar{\sigma}$ is a critical point of $P^d$ then we have
\begin{eqnarray}\label{devsig}
P^d(\bar{\sigma})'&=&\left[\bar{c}^2-2x\bar{c}+x^2- \phi^{-1}\left(\bas\right)\right]G'(\bar\sigma)-\nonumber\\
&&\sigma\left[\phi'\left(\phi^{-1}\left(\bas\right)\right)\left(\phi^{-1}\left(\bas\right)\right)'-1\right]=0.
\end{eqnarray}
Notice that
\begin{equation}
\left(\phi^{-1}\left(\bas\right)\right)'=\frac{1}{\phi'\left(\bar\epsilon\right)}=\frac{1}{\phi'\left(\phi^{-1}\left(\bas\right)\right)},
\end{equation}
The third term in (\ref{devsig}) is zero. The term $G'(\bar\sigma)$ is not zero from the hypothesis, so  we obtain
\begin{equation}\label{devsig1}
(x-\bar{c})^2-\phi^{-1}\left(\bas\right)=0,
\end{equation}
that is
\begin{equation}\label{tsig}
\bar{\sigma}=w\phi\left(\|x-\bar{c}\|^2\right)-y.
\end{equation}
The critical point condition for the primal problem $P'(c)=0$ is
\begin{equation}
 -2 w(x-c)\phi'(\|x-c\|^2)(w\phi(\|x-c\|^2)-y)+\beta c-f=0.
\end{equation}
By considering that $\phi'(\|x-c\|^2)=\phi'\left(\phi^{-1}\left(\bas\right)\right)$ and $\sigma=w\phi\left((x-{c})^2\right)-y$ we obtain
\begin{equation}\label{derc}
2w(x-c)\phi'\left(\phi^{-1}\left(\s\right)\right)\sigma+\beta c-f=0,
\end{equation}
that is
\begin{equation}\label{derc1}
c=\frac{2\phi'\left(\phi^{-1}\left(\s\right)\right)\sigma+f}{2\phi'\left(\phi^{-1}\left(\s\right)\right)\sigma+\beta}.
\end{equation}
By setting $\sigma=\bar{\sigma}$ in (\ref{derc1}) we obtain (\ref{c}) proving that $\bar{c}$ is a critical point of $P(c)$.

\smallskip\noindent
For the correspondence of the function values we start from the dual function
\begin{eqnarray}
P^d(\bar{\sigma})&=& - \frac{1}{2}\frac{F^2(\bar\sigma)}{G(\bar\sigma)}-U^*(\bar{\sigma})\bar{\sigma}-V^*(\bar{\sigma})+\nonumber\\
&& x^2w\phi'\left(\phi^{-1}\left(\bas\right)\right) \bar{\sigma}
\end{eqnarray}
add and subtract the term $\frac{1}{2}\frac{F^2(\bar\sigma)}{G(\bar\sigma)} $ and substitute the value of $\bar{c}$
\begin{eqnarray}
&\frac{1}{2} \bar{c}^2 G(\bar{\sigma})-\bar{c} F(\bar{\sigma})-U^*(\bar{\sigma})\bar{\sigma}-V^*(\bar{\sigma})+&\nonumber\\
 &x^2w\phi'\left(\phi^{-1}\left(\bas\right)\right) \bar{\sigma}&
\end{eqnarray}
by reordering the terms we obtain
\begin{eqnarray}
 &=&\left(\|x-\bar{c}\|^2w\phi'\left(\phi^{-1}\left(\bas\right)\right)-U^*(\bar{\sigma})\right)\bar{\sigma}\nonumber\\
 &&-V^*(\bar{\sigma})+\frac{1}{2}\beta\bar{c}^2-f\bar{c},
\end{eqnarray}
Considering the (\ref{dmap}), setting $\bar{\epsilon}=\|x-\bar{c}\|^2$ and   $\phi'\left(\phi^{-1}\left(\bas\right)\right)=\phi'(\bar{\epsilon})$ we obtain:
\begin{eqnarray}
 \left[w \phi'(\bar{\epsilon})\bar{\epsilon}-w\phi'\left(\bar{\epsilon}\right)\bar{\epsilon}+w \phi(\bar{\epsilon})\right]\left[w\phi(\bar{\epsilon})-y\right]-&&\nonumber\\
 \frac{1}{2}(w\phi(\bar{\epsilon})-y)^2+y(w\phi(\bar{\epsilon})-y)+\frac{1}{2}\beta\bar{c}^2-f\bar{c}&=& \nonumber\\
w^2\phi(\bar{\epsilon})^2-y w\phi(\bar{\epsilon})-\frac{1}{2}(w\phi(\bar{\epsilon})-y)^2&& \nonumber\\
-y w\phi(\bar{\epsilon})+y^2+\frac{1}{2}\beta\bar{c}^2-f\bar{c}&&
\end{eqnarray}
by collecting the terms we obtain:
\begin{equation}
(w\phi(\bar{\epsilon})-y)^2-\frac{1}{2}(w\phi(\bar{\epsilon})-y)^2+\frac{1}{2}\beta\bar{c}^2-f\bar{c},
\end{equation}
that is
\begin{equation}
\frac{1}{2}\left(w\phi({\|x-\bar{c}\|^2})-y\right)^2+\frac{1}{2}\beta\bar{c}^2-f\bar{c}=P(\bar{c}).
\end{equation}
that proves the theorem. \hfill  \qed
}
\end{proof}

Theorem \ref{criteq} shows that the problem $({\cal P}^d)$ is canonically dual to the primal $({\cal P})$ in the sense that the duality gap is zero.

\section{Gaussian function}

One of the most used RBF is the Gaussian function. In  this section we will analyze the problem with $\phi(\|x-c\|^2)=\exp\left\{-\frac{\|x-c\|^2}{2\alpha^2}\right\}$, where $\alpha$ is a parameter that represents the standard deviation of the Gaussian function. In the RBFNN formulation normally there is no the linear term  $fc$. The primal problem is:
\begin{equation}\label{exp}
\min P(c)=\frac{1}{2} \left(w\exp\left\{-\frac{\|x-c\|^2}{2\alpha^2}\right\}-y\right)^2+\frac{1}{2} \beta c^2
\end{equation}
If we define the quantity $d(c)=\frac{\|x-c\|^2}{2\alpha^2}$, the nonlinear operator $\xi:\real\rightarrow \cal{E}\rm_a$  from (\ref{operator}) becomes
\begin{equation}
\xi=w\exp\left\{\dis\right\}.
\end{equation}
The expressions that define $\sigma$, $V$ and $V^*$ are the same as the general problem that is:
\begin{itemize}
\item $V(\xi(c))=\frac{1}{2}(\xi-y)^2$;
\item $\sigma= \xi-y$;
\item $V^*(\sigma)= \left(\frac{1}{2}\sigma^2+y\sigma\right)$.
\end{itemize}
The second order operator $\Lambda_2(c):\real \rightarrow \cal{E}\rm_b$ is
\begin{equation}
\epsilon=\Lambda_2(c)=\|x-c\|^2=\epsilon
\end{equation}
The second level canonical function becomes
\begin{equation}
U(\epsilon)=w\exp\left\{-\frac{\epsilon}{2\alpha^2}\right\}.
\end{equation}
And the second order duality mapping $\tau$ is
\begin{equation}
\tau=w\phi'(\epsilon)=-\frac{w}{2\alpha^2} \exp\left\{-\frac{\epsilon}{2\alpha^2}\right\}.
\end{equation}
So the Legendre conjugate $U^*: \cal{S}'\rm_b\rightarrow \real$ is
\begin{eqnarray}
U^*(\tau)&=& \tau\left(\phi^{-1}\left( \frac{\tau}{w}\right )\right)'-w\phi \left(\phi^{-1}\left( \frac{\tau}{w}\right )\right)'\nonumber\\
&=&-2\alpha^2\tau\left(\ln\left(\frac{-2\alpha^2\tau}{w}\right)-1\right).
\end{eqnarray}
The derivative of the exponential function is the exponential function itself. This simplifies the relation (\ref{relation1}) between $\xi$ and $\tau$ making it linear, that is $\xi=-\frac{\tau}{2\alpha^2}$. The relation between $\sigma$ and $\tau$ is:
\begin{equation}
\tau=-\frac{(\sigma+y)}{2\alpha^2}
\end{equation}
that is also linear. The total complementarity function becomes:
\begin{eqnarray}
\Xi(c,\sigma)&=& \frac{1}{2} c^2G(\sigma)-cF(\sigma)-U^*(\sigma)\sigma-V^*(\sigma)-\nonumber\\
&&\frac{x^2(\sigma^2+y\sigma )}{2\alpha^2}
\end{eqnarray}
where:
\begin{eqnarray*}
G(\sigma)&=&\beta- \frac{\sigma^2+  y \sigma}{\alpha^2}\\
F(\sigma)&=&- \frac{x\sigma^2+ x y \sigma}{\alpha^2}\\
U^*(\sigma)&=& \left(\sigma+y\right)\left(\lns-1\right)\\
\s&=&\frac{\sigma+y}{w}
\end{eqnarray*}
The dual problem is
\begin{eqnarray}\label{duexp}
P^d(\sigma)&=&-\frac{1}{2}\frac{F(\sigma)^2}{G(\sigma)}-  \lns\left(\sigma^2+ y\sigma\right)+\frac{1}{2}\sigma^2 \nonumber\\
&&-\frac{x^2(\sigma^2+y\sigma)}{2\alpha^2}
\end{eqnarray}
The domains of the variables in the primal and dual problems are:
\begin{itemize}
\item $\cal{E}\rm_b=\{\epsilon \in \real | \epsilon\ge 0\}$
\item $\cal{S}\rm_b=\{ \tau \in \real| -\infty<\tau<0\}$ if $w>0$, $\cal{S}\rm_b=\{ \tau \in \real| -\infty<\tau<0\}$ if $w<0$
\item $\cal{E}\rm_a=\{\xi\in\real|0\le\xi\le w$\}
\item $\cal{S}\rm_a=\{\sigma \in \real|   -y\le\sigma\le w-y \}$ if $w>0$, $\cal{S}\rm_a=\{\sigma \in \real| w-y  \le\sigma\le-y \}$ if $w<0$
\end{itemize}

\begin{remark}
Parameters $\beta$, $x$, $y$, and $w$ play  important roles in solving the non-convex problem (P). In the original problem (\ref{primal}) one searches for the value of $c$ that brings the term $w \exp \left\{\dis\right\}$ as closer as possible to $y$, that is $\sigma=w\exp\left\{\dis\right\}-y=0$.

\noindent
If $y<0$ and $w>0$ or $y>0$ and $w<0$  we will have that $|\sigma|>0$. This means that in the case of the exponential function, it would be better to choose $c$ as bigger as possible in order to make the exponential go to zero, but the result would never be satisfactory as the error committed by the approximation would go close to $-y$ as $c$ goes to infinity. The value $-y$ is not a good value for the error as it is far from zero. On the other hand if $y$ and $w$ have the same sign and $|y|>|w|$ the value of $c$ will be $x$  in order to have the exponential equal to $1$ and to have the lowest value for $\sigma=w \exp\left\{\dis\right\}-y$.

\noindent
In order to have a realistic problem, we will consider the case with $y$ and $w$ with the same sign, and with $|y|<|w|$. The cases with $y,w>0$ and $y,w<0$ are equivalent, so we will suppose that both $y$ and $w$ are positive without losing generality.
\end{remark}

\begin{theorem}\label{sign}
Suppose that $\bar{\sigma}\in \cal{S}\rm_a$ is a critical point of the dual problem (\ref{duexp}) with the corresponding $\bar{c}= \frac{F(\bar{\sigma})}{G(\bar{\sigma})}\in \real$ and that $\bar\sigma\ne \frac{y}{2}$. Then $\bar{c}$ is a critical point of the primal problem and:
\begin{equation}\label{eq}
P^d(\bar{\sigma})=P(\bar{c}).
\end{equation}
moreover, there are the following relations between the critical points of the primal problem and the dual problem:
\begin{enumerate}
\item If $(2\bar\sigma+y)>0$ and $G(\bar\sigma)\ge 0$ or $(2\bar\sigma+y)<0$ and $G(\bar\sigma)\le 0$ then if $\bar\sigma$ is a local minimum of the dual problem, the corresponding $\bar{c}$ is a local maximum of the primal problem; if $\bar\sigma$ is a local maximum of the dual problem the corresponding $\bar{c}$ is a local minimum of the primal problem;
\item If $(2\bar\sigma+y)>0$ and $G(\bar\sigma)\le 0$ or $(2\bar\sigma+y)<0$ and $G(\bar\sigma)\ge 0$ then if $\bar\sigma$ is a local minimum of the dual problem the corresponding $\bar{c}$ is a local minimum of the primal problem; if $\bar\sigma$ is a local maximum of the dual problem the corresponding $\bar{c}$ is a local maximum of the primal problem.
\end{enumerate}
Let $x_o=\sqrt{-2 \alpha^2\ln\left(\frac{y}{2w} \right)}$. If $\bar\sigma=-\frac{y}{2}$, then there is a corresponding critical point to $\bar\sigma$ in the primal problem if and only if the parameters $x$, $y$, $\beta$ and $w$ satisfy one of the two following conditions:
\begin{equation}\label{critf}
\begin{array}{c}
\beta x +\left(\beta+\frac{y^2}{4 \alpha^2}\right) x_o=0\\
\beta x -\left(\beta+\frac{y^2}{4 \alpha^2}\right) x_o=0
\end{array}
\end{equation}
and the corresponding critical point $\bar{c}$ in the primal problem is always a local minimum. If neither of conditions (\ref{critf}) is satisfied, $\bar\sigma=-\frac{y}{2}$ is always a critical point of the dual problem, but it does not have any corresponding critical point in the primal problem.
\end{theorem}

\begin{proof} {\em 
The first order derivative for the dual problem is:
\begin{equation}\label{dud}
P^d(\sigma)'=-\left[\left(x-\frac{ F(\sigma)}{G(\sigma)}\right)^2 \frac{1}{2 \alpha^2}+\lns\right]\left[2\sigma+y\right]
\end{equation}
so the term (\ref{taudev}) is equal to $2\bar\sigma+y$. If $\bar\sigma\neq -\frac{y}{2}$, the critical point equivalency and condition (\ref{eq}) are consequences of Theorem \ref{criteq}.

\noindent
To prove statements $(i)$ and $(ii)$ we use the second order derivatives of the problems $P(c)$ and $P^d(\sigma)$
\begin{eqnarray}\label{pdd}
&P(c)''=\frac{(x-c)^2}{\alpha^4}\exp\left\{\dis\right\}\left(2w\exp\left\{\dis\right\}-y\right)\nonumber&\\
&+\beta-\frac{1}{\alpha^2}w\exp\left\{\dis\right\}\Big(w\exp\left\{\dis\right\}-y\Big)&
\end{eqnarray}
\begin{eqnarray}\label{dudd}
P^d(\sigma)''&=&-  \frac{1}{\alpha^2}\left(x-\frac{F(\sigma)}{\sigma}\right)^2\left(1+\frac{(2\sigma+y)^2}{\alpha^2G(\sigma)}\right)  \nonumber\\
&&-\frac{2\sigma+y}{\sigma+y}-2\lns
.
\end{eqnarray}
Since $\bar\sigma$ is a critical point of the dual, we have that  $P^d(\sigma)'=0$. Therefore when $\bar\sigma\neq -\frac{y}{2}$:
\begin{equation}\label{cricon}
\left(x-\frac{ F(\bar\sigma)}{G(\bar\sigma)}\right)^2 =-2 \alpha^2\lnbas
\end{equation}
By using condition (\ref{cricon}) in (\ref{dudd}) we obtain:
\begin{equation}\label{cridudd}
P^d(\bar\sigma)''=(2\bar\sigma+y)\left(\frac{2\lnbas(2\bar\sigma+y)}{\alpha^2G(\bar\sigma)}  -\frac{1}{\bar\sigma+y}\right).
\end{equation}
Noticing  $\sigma=w \exp\left\{\dis\right\}-y$, it is possible to rewrite $P(\bar{c})''$ in terms of $\bar\sigma$, i. e.:
\begin{equation}\label{pdds}
P(c(\bar\sigma))''=G(\bar\sigma)+\frac{2}{\alpha^2}(\bar\sigma+y)(2\bar\sigma+y)\left(x-\frac{F(\bar\sigma)}{G(\bar\sigma)}\right)^2.
\end{equation}
by using again condition (\ref{cricon}) we obtain:
\begin{equation}
P(c(\bar\sigma))''=\frac{1}{\alpha^2}\left[\alpha^2G(\bar\sigma)-2(\bar\sigma+y)(2\bar\sigma+y)\lnbas\right]
\end{equation}
so it is possible to rewrite equation (\ref{cridudd}) in the following form:
\begin{equation}\label{relation}
P^d(\bar\sigma)''=-\frac{2\bar\sigma+y}{G(\bar\sigma)(\bar\sigma+y)}P(c(\bar\sigma))''.
\end{equation}
and to find the relations reported in Table 1. From these relations, we obtain:
\noindent

\begin{itemize}
\item If $(2\sigma+y)>0$ and $G(\sigma)\ge 0$ or $(2\sigma+y)<0$ and $G(\sigma)\le 0$ then the second order derivate of the primal problem and the second order derivate of the dual problem have opposite sign at their critical points;
\item If $(2\sigma+y)>0$ and $G(\sigma)\le 0$ or $(2\sigma+y)<0$ and $G(\sigma)\ge 0$ then the second order derivate of the primal problem and the second order derivate of the dual problem have the same sign at their critical points.
\end{itemize}
This proves statements $1$ and $2$.

\begin{table}[ht]
\begin{center}
\begin{tabular}{c c c c}
\hline
$(2\bar\sigma+y)$ & $G(\bar\sigma)$ & $P(c(\bar\sigma))$ & $P^d(\bar\sigma)$\\
\hline
$>0$ &$>0$&$\pm$&$\mp$\\
$>0$&$<0$&$\pm$&$\pm$\\
$<0$&$<0$&$\pm$&$\mp$\\
$<0$&$>0$&$\pm$&$\pm$\\
\hline
\end{tabular}
\end{center}
\caption{Relations between the second order derivatives of the primal problem and dual problem}
\end{table}

\smallskip\noindent
The point $\bar\sigma=-\frac{y}{2}$ is a critical point of $P^d$ according to the second part of the (\ref{dud}). The point $\bar{c}$ corresponding to $\bar\sigma=-\frac{y}{2}$  is a critical point of the primal problem if and only if $P'(\bar{c})=0$. We can use the (\ref{dmap}) to find the relation between $\bar\sigma$ and $\bar{c}$ that is:
\begin{equation}
\bar\sigma=\bar\xi-y \rightarrow \bar\sigma= w \exp\{\disb\} -y
\end{equation}
\begin{equation}
\bar{c}=x\pm\sqrt{-2 \alpha^2\left(\lnbas\right)}.
\end{equation}
For $\bar\sigma=-\frac{y}{2}$ we obtain:
\begin{equation}
\bar{c}=x\pm x_o.
\end{equation}
Substituting these values in the first order derivative of the primal problem:
\begin{equation}\label{pd}
P'(\bar{c})=\frac{1}{2}d(\bar{c})w\exp\{\disb\}\left( w\exp\{\disb\} -y \right)+ \beta \bar{c}
\end{equation}
and considering that $w \exp \left\{\disb\right\}=\bar\sigma+y=\frac{y}{2}$ and $w \exp\left\{\disb\right\} -y =\bar\sigma=-\frac{y}{2}$ we obtain that the primal problem has a critical point at $\bar{c}$  corresponding to the critical $\bar\sigma=-\frac{y}{2}$ if and only if:
\begin{equation}
\beta x \pm\left(\beta+\frac{y^2}{4 \alpha^2}\right) x_o=0.
\end{equation}
This happens only for a particular configuration of the parameters $w$, $\beta$, $x$ and $y$ that makes one of the roots the first term of the derivative (\ref{dud}):
\begin{equation}
-\left[\left(x-\frac{F(\bar\sigma)}{G(\bar\sigma)}\right)^2 \frac{1}{2 \alpha^2}+\left(\lnbas\right)\right]=0
\end{equation}
be in $\bar\sigma=-\frac{y}{2}$.

\noindent
To prove that at $\bar\sigma=-\frac{y}{2}$ the critical point of the dual problem corresponds to a minimum point of the primal problem we plug the value of $\bar\sigma=-\frac{y}{2}$ in the (\ref{pdds}) and obtain
\begin{equation}
P''(\bar\sigma)= \beta+\frac{y^2}{4\alpha^2},
\end{equation}
which is always a positive value. \hfill  \qed
}
\end{proof}

\begin{remark}
From now on we will refer to the critical point  $\sigma_f=-\frac{y}{2}$ as pseudo dual critical point  as it is a critical point of the dual problem that generally does not have a corresponding critical point for the primal problem.
\end{remark}

\subsection{Choice of the critical point}
\noindent
In order to find the best solution among the critical points of problem (\ref{exp}) we introduce the following feasible spaces:
\begin{equation}\label{s+}
\cal{S}\rm^+_a=\{\sigma \in \cal{S}\rm_a | G(\sigma)>0\}
\end{equation}
\begin{equation}\label{s-}
\cal{S}\rm^-_a=\{\sigma \in \cal{S}\rm_a | G(\sigma)<0\}
\end{equation}
The following theorem explains the relations between the critical points:
\begin{theorem}\label{ins+}
Suppose that the point $\bar\sigma_1\in\cal{S}\rm^+_a $ and $\bar\sigma_2 \in \cal{S}\rm^-_a$ are critical points of the dual problem, that $\bar\sigma_i\ne-\frac{y}{2}$ for $i=1, 2$ and that $\bar{c}_1$ and $\bar{c}_2$ are the corresponding critical points of the primal problem. Then if both $\bar{c}_1$ and $\bar{c}_2$ are local minima  or local maxima of the primal problem, the following relation always holds:
\begin{equation}\label{order}
P(\bar{c}_1)=P^d(\bar\sigma_1)<P(\bar{c}_2)=P^d(\bar\sigma_2)
\end{equation}
\end{theorem}
\begin{proof}
{\em This theorem is a consequence of the first theorem in triality theory \cite{gaob}. \hfill  \qed
}
\end{proof}

\begin{remark}
The pseudo critical point $\sigma_f=-\frac{y}{2}$ is always in $S_a^+$.
\end{remark}

\noindent
From the results in Theorem \ref{ins+} it is always better to search for the dual critical point in $\cal{S}\rm^+_a$ that corresponds to a minimum in the primal problem. In order to characterize the solutions in $\cal{S}\rm^+_a$ and the domains in which search for the best solution, two theorems are proposed in the following:

\begin{theorem}\label{fakesign}
Let $\sigma_f=-\frac{y}{2}$ be the pseudo critical point of the dual problem, $x_o=\sqrt{-2 \alpha^2\ln\left(\frac{y}{2w}\right)}$, $x$ positive. Then:
\begin{itemize}
\item if $x\in \left(0,x_o\right)$ then $\sigma_f$ is always a local minimum of $P^d(\sigma)$;
\item if $x>x_o$ then:
\begin{enumerate}
\item if $\beta>0$ and $\beta< \frac{y^2x_o}{4\alpha^2 \left(   x-x_o\right)}$, $\sigma_f$ is a local minimum for the dual problem;
\item if $\beta>0$ and $\beta> \frac{y^2x_o}{4\alpha^2 \left(   x-x_o\right)}$, $\sigma_f$ is a local maximum for the dual problem;
\item if $\beta>0$, $\beta=\frac{y^2x_o}{4\alpha^2 \left(   x-x_o\right)}$, $\sigma_f$ is an inflection point in which the first order derivative is zero and that corresponds to a a local minimum of the primal problem.
\end{enumerate}
\end{itemize}
\end{theorem}

\begin{proof}
{\em In order to understand that $\sigma_f=-\frac{y}{2}$ is  a minimum or a maximum for the dual we have to plug its value in  the second order derivative of $P^d(\sigma)$ that is equation (\ref{dudd}) and analyze its sign. After the substitution we obtain
\begin{equation}\label{dudds}
P^d(\sigma_f)=-\left[2 \ln\left(-\frac{y}{2w}\right)+\frac{1}{\alpha^2}\left(\frac{x \beta}{\beta+\frac{y^2}{4 \alpha^2}}\right)^2\right].
\end{equation}
The first order derivate in $\beta$ of (\ref{dudds}) is $-\frac{2x\beta^2}{\alpha^2\left( \beta+ \frac{y^2}{4\alpha^2}\right)^2}$, that is the function is monotonic decreasing in $\beta$. The value of (\ref{dudds}) in $\beta=0$ is $-\ln\left(-\frac{y}{2w}\right)$ that is positive. If we make $\beta$ go to $+\infty$ we obtain:
\begin{equation}
\lim_{\beta\to+\infty} -\left[2 \ln\left(-\frac{y}{2w}\right)+\frac{1}{\alpha^2}\left(\frac{x \beta}{\beta+\frac{y^2}{4 \alpha^2}}\right)^2\right]=-2\ln\left(-\frac{y}{2w}\right)+\frac{x^2}{\alpha^2}
\end{equation}
that is the second order derivative of $P^d(\sigma)$ in $\sigma_f$ is non negative for any value of $\beta>0$ if
\begin{equation}
x\in \left[-x_o,x_o\right]
\end{equation}
If $x$ does not satisfy this condition, from the (\ref{dudds}) we have that the second order derivative of the dual problem is positive in $\sigma_f$ if $\beta$ satisfies:
\begin{equation}\label{duddcon}
\beta> \frac{-y^2x_o}{4\alpha^2 \left(   x+x_o\right)}
 \mbox{ and }\beta< \frac{y^2x_o}{4\alpha^2 \left(   x-x_o\right)}.
\end{equation}
On the other hand if:
\begin{equation}\label{duddcon2}
\beta<  \frac{-y^2x_o}{4\alpha^2 \left(   x+x_o\right)}
 \mbox{ or }\beta> \frac{y^2x_o}{4\alpha^2 \left(   x-x_o\right)}
 \end{equation}
there will be a local maximum in $\sigma_f$. As $x$ is considered positive, the term $ \frac{-y^2x_o}{4\alpha^2 \left(   x+x_o\right)}$ is always negative, so $\beta$ will always be greater than it.

\noindent
If the condition $\beta=\frac{y^2x_o}{4\alpha^2 \left(   x-x_o\right)}$ is satisfied, the critical point $\sigma_f$ is an inflection point that also satisfies the first order condition and it has a corresponding minimum point in the primal problem for Theorem \ref{sign}. \hfill  \qed}
\end{proof}

\begin{remark}
In the case of $x$ negative, the conditions are changed in the following way:
\begin{itemize}
\item if $x\in \left(-x_o,0\right)$ then $\sigma_f$ is always a local minimum of $P^d(\sigma)$
\item if $x<-x_o$ then:
\begin{enumerate}
\item if $\beta>0$ and $\beta< \frac{-y^2x_o}{4\alpha^2 \left(   x+x_o\right)}$, $\sigma_f$ is a local minimum for the dual problem;
\item if $\beta>0$ and $\beta> \frac{-y^2x_o}{4\alpha^2 \left(   x+x_o\right)}$, $\sigma_f$ is a local maximum for the dual problem;
\item if $\beta>0$, $\beta=\frac{-y^2x_o}{4\alpha^2 \left(   x+x_o\right)}$, $\sigma_f$ is an inflection point in which the first order derivative is zero and that corresponds to a a local minimum of the primal problem.
\end{enumerate}
\end{itemize}
The proof of these statement is similar to that of Theorem \ref{fakesign} and can be omitted.  
\end{remark}

\begin{remark}
Theorem  \ref{fakesign} shows the effects of the parameter $\beta$ on the pseudo critical point $\sigma_f$. Similar effects can also be obtained in respect to $y$, $x$, $\alpha$, and $w$. The reason we choose  $\beta$ is because it is an hyper-parameter that can be chosen by the practitioner before performing the optimization.
\end{remark}
For the next theorem, we introduce the two following subsets of $\cal{S}\rm^+_a$:
\begin{equation}
{\cal S}^+_\sharp=\left\{\sigma\in{\cal S}^+_a | \sigma>-\frac{y}{2}\right\}
\end{equation}
\begin{equation}
{\calS}^+_\flat=\left\{\sigma\in{\calS}^+_a | \sigma<-\frac{y}{2}\right\}
\end{equation}
\begin{theorem}\label{s-}
Let $\sigma_f=-\frac{y}{2}$ be the pseudo critical point in the dual problem and let the primal problem have a maximum of five critical points. Then
\begin{itemize}
\item if $\sigma_f$ is a local minimum for the dual function, there will be a local maximum in $\cal{S}\rm^+_\sharp$ that corresponds to a minimum of the primal problem.
\item if $\sigma_f$ is a local maximum then:
\begin{enumerate}
\item there are no critical points in  $\cal{S}\rm^+_\sharp$;
\item there is at least one critical point in $(\cal{S}\rm^+_\flat$
\end{enumerate}
\end{itemize}
\end{theorem}


\begin{proof} {\em 
\noindent
In the dual problem there must be a singularity point in $G(\sigma)=0$ that goes to $-\infty$, so if $\sigma_f$ is a local minimum, there must be a local maximum in $\cal{S}\rm^+_\sharp$.

\noindent
If $\sigma_f$ is a local maximum, we prove condition $(i)$ by negating the thesis and suppose that there is a least one critical point in $\cal{S}\rm^+_\sharp$. As $P^d(\sigma)$ goes to $-\infty$ if $G(\sigma)\rightarrow 0$, there will be no one, but two critical points in $\cal{S}\rm^+_\sharp$, a local minimum $\sigma_1$ and a local maximum $\sigma_2$ with the relation $P^d(\sigma_1)<P^d(\sigma_2)$. For Theorems \ref{sign} and \ref{ins+}, $\sigma_1$ corresponds to the second highest local maximum of the primal function $c_1$, and $\sigma_2$ corresponds to the lowest or second lowest local minimum of the primal function $c_2$, that is the relation $P(c_2)<P(c_1)$ is satisfied. By Theorem \ref{criteq} we have:
\begin{equation}
P^d(\sigma_1)<P^d(\sigma_2)=P(c_2)<P(c_1)=P^d(\sigma_1)
\end{equation}
that is a contradiction.

\noindent
To prove condition $(ii)$, it is sufficient to notice that if there are no critical points in $\cal{S}\rm^+_\sharp$, for the triality theory there must be at least one critical point corresponding to the global minimum in $\cal{S}\rm^+_a$ and this point will be in $\cal{S}\rm^+_\flat$. \hfill  \qed }
\end{proof}
\vspace*{13pt}

\begin{figure}[ht]
\begin{center}
\includegraphics[scale=.49]{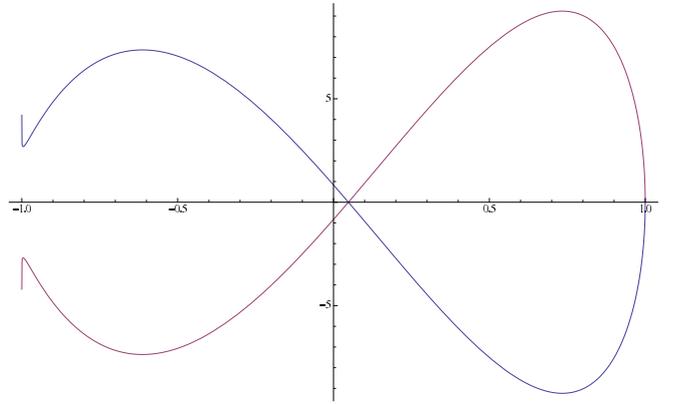}
\caption{Dual algebraic curves with $y=1$,  $w=2$, $\alpha=\frac{\sqrt{2}}{2}$ and $\beta=0.1$ in respect to the internal input $x$}
\label{dc}
\end{center}
\end{figure}

\begin{figure}[h]
\begin{center}
\includegraphics[scale=.28]{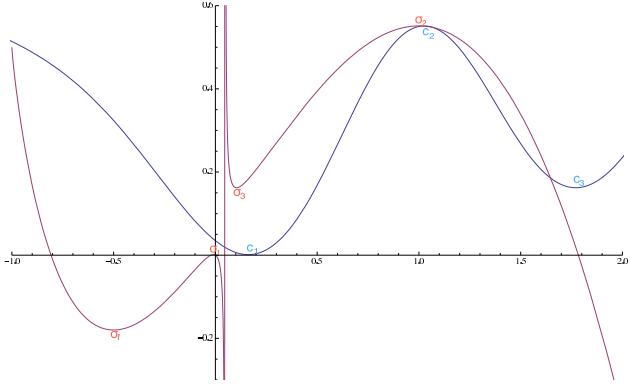}
\caption{Primal(in blue) and dual(in red) functions for Case 1 with three critical points}
\label{3r}
\end{center}
\end{figure}
Depending on the parameters, the primal problem (\ref{exp}) can have at most five critical points. There are several cases:

  {\bf Case 1}: Three critical points for $P(c)$ and four critical points for $P^d(\sigma)$, two critical point in $\cal{S}\rm^+_a$ and two critical points in $\cal{S}\rm^-_a$, with $\sigma_f$ as local minimum. The values of the parameters are $y=1$, $x=1$, $w=2$, $\alpha=\frac{\sqrt{2}}{2}$, $\beta=0.1$ (see Figure \ref{3r}). This case can be easily solved with the general canonical duality framework\cite{gaob}, as the local maximum in $\cal{S}\rm^+_a$ corresponds to the global minimum of the problem, and the local minimum and maximum in $\cal{S}\rm^-_a$ correspond to the local minimum and maximum in the primal problem.
\vspace*{13pt}

\begin{figure}[h]
\begin{center}
\includegraphics[scale=.29]{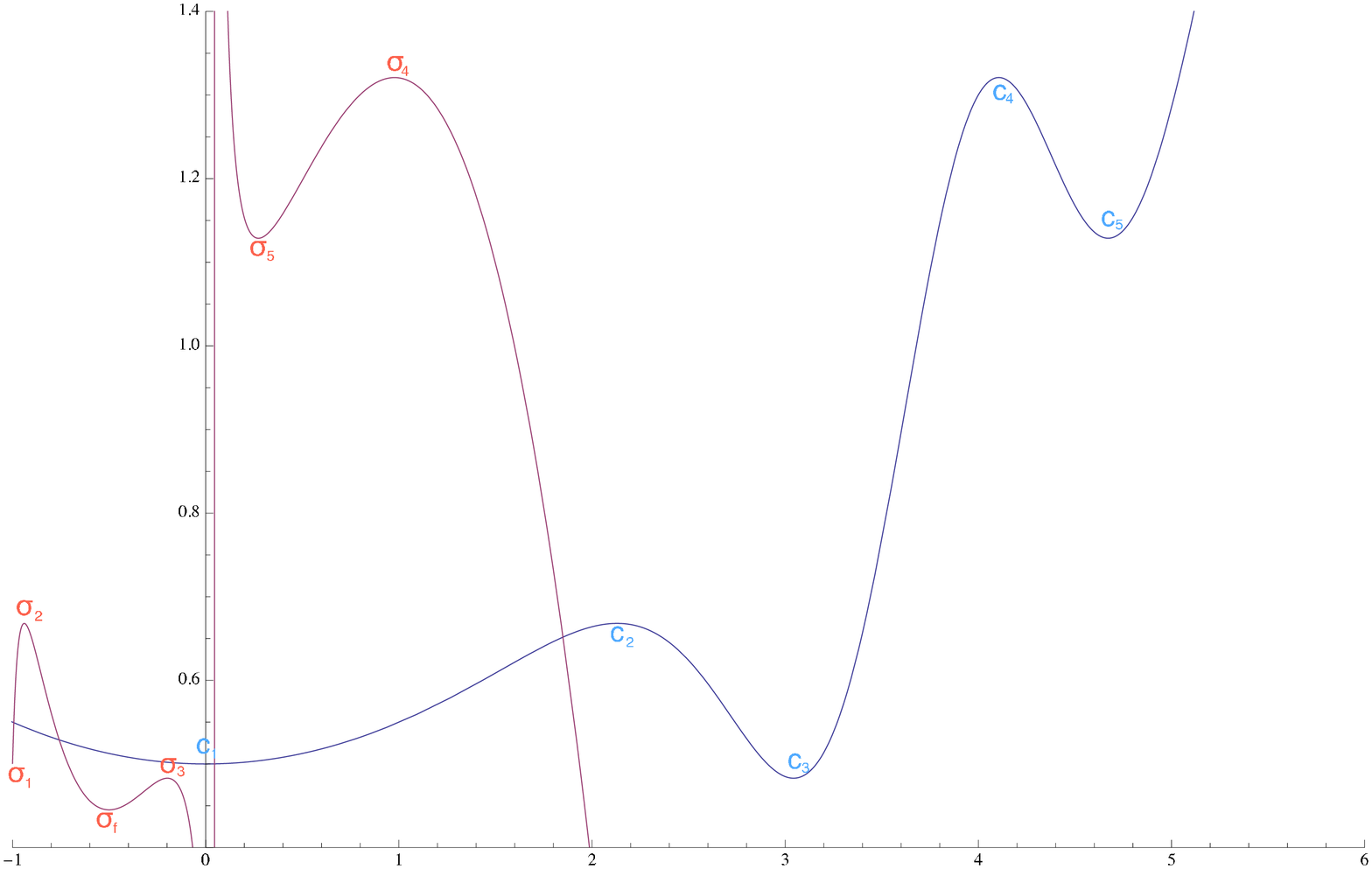}
\caption{Primal(in blue) and dual(in red) functions for Case 2 with five critical points in the primal and six critical points in the dual.}
\label{5r}
\end{center}
\end{figure}

\begin{figure}[h]
\begin{center}
\includegraphics[scale=.44]{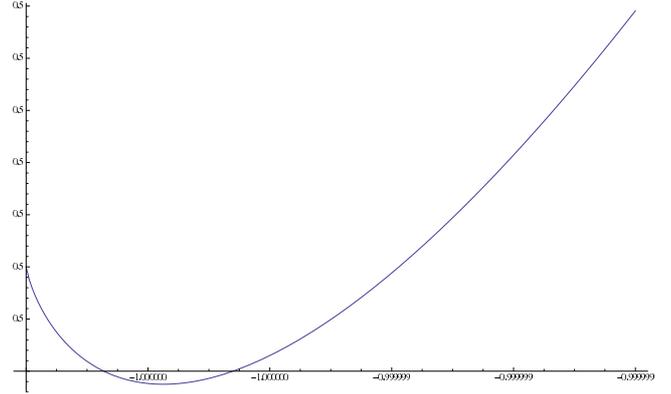}
\caption{Critical point on the boundary of the dual function feasible set for Case 2.}
\label{5min}
\end{center}
\end{figure}

\begin{figure}[h]
\begin{center}
\includegraphics[scale=.29]{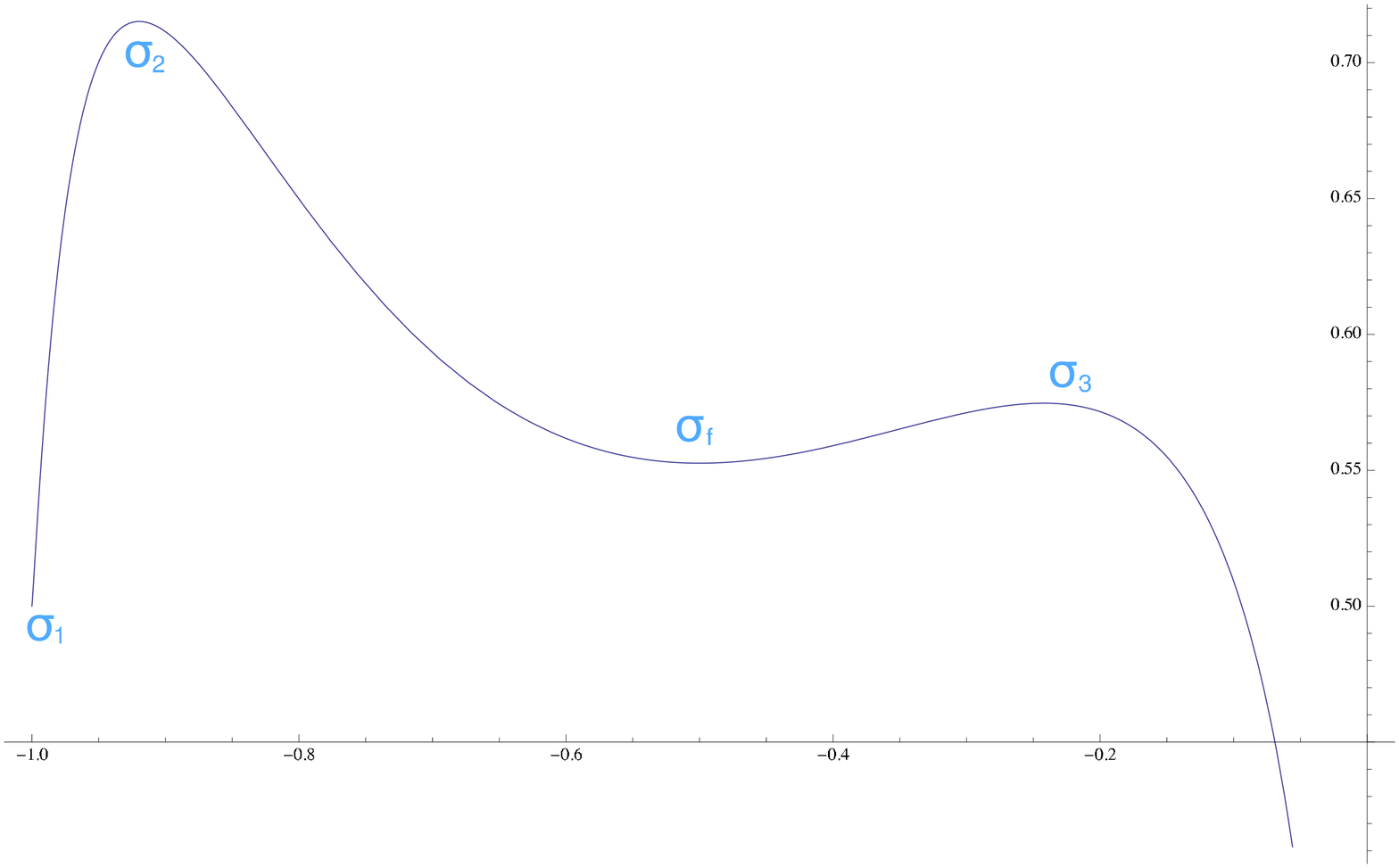}
\caption{$\cal{S}\rm^+_a$ of the dual problem in the case of $\beta=0.12$. The minimum near the boundary $\sigma_1$ is a global minimum.}
\label{5b12}
\end{center}
\end{figure}

{\em\bf Case 2}: Five critical points for $P(c)$,  six critical points for $P^d(\sigma)$. The values of the parameters are  $y=1$, $x=4$, $w=2$, $\alpha=\frac{\sqrt{2}}{2}$ and $\beta=0.1$ (see Figure \ref{5r}). Notice that the only parameter that changed in respect to Case $1$ is $x$. With these parameters the problem becomes multi-welled. The two critical points with the lowest value of the objective function belong to the same double well and their corresponding critical points are in $\cal{S}\rm^+_a$.
 The critical point $\sigma=-0.999999$  of  $P^d(\sigma)$ is  corresponding to the second best minimizer  $c=0.00002$  of the primal problem and this $\sigma$   is situated near the boundary of ${\cal S}^+_{\flat}$
 which is visible in Figure \ref{5min}.
\noindent
It is also possible, for certain values of the parameters, that the local minimum on the boundary of $S_a$, corresponds to the global minimum of the problem (see Figure \ref{5b12}).
\noindent
In this case the choice of the value for $\sigma$ should be the critical point near the boundary. This critical point  corresponds to a critical point in the primal with the value of $c$ near zero. This critical point is generated by the term $\frac{1}{2} \beta  c^2$ that is the regularization term used to make the objective function coercive and more regular. On the other hand, this term doesn't have anything to do with the original aim of the problem. This point near zero in the primal function will always have the corresponding dual critical point near the boundary, because as $c$ gets close to zero, $\sigma=w\exp\left\{\dis\right\}-y$ gets close to $-y$. We also consider that $\sigma=w\exp\left\{\dis\right\}-y$ is the error that originally we want to minimize in problem (\ref{W1}) and that the critical point on the boundary will always have a $\sigma$ with an absolute value bigger than the other critical point closer to $\sigma=0$. In other words the local minimum on the boundary has nothing to do with the original problem, has an high value of the error and should not be considered as a good solution. In order to find the optimal solution for the original problem, the local minimum in the primal problem corresponding to the critical point closer to zero in $S_a^+$ is preferable. By reducing the value of $\beta$ it is possible not only to make the critical point near $c=0$ into a local minimum, but also to assure that $\sigma_f$ is a local minimum. In this way there is a critical point in $\cal{S}\rm^+_\sharp$ and the domain of the solution is well defined. Basically if the critical point near the boundary of $S_a^+$ is the global minimum, a very big value of $\beta$ has been chosen.

\begin{figure}[h]
\begin{center}
\includegraphics[scale=.42]{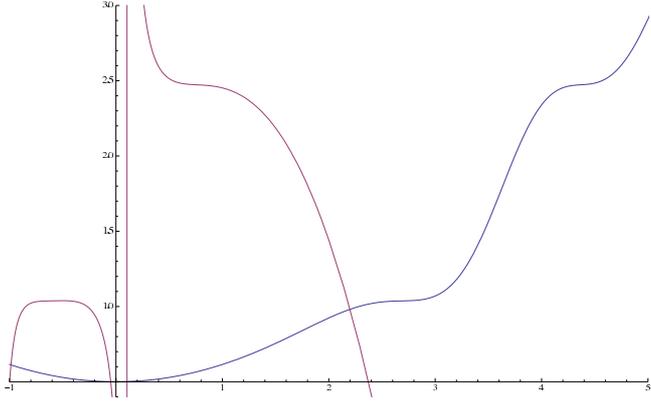}
\caption{Primal(in blue) and dual(in red) functions for the Case 3 with three critical points in the primal and four critical points in $\cal{S}\rm^+_a$.}
\label{3b.22}
\end{center}
\end{figure}

{\em\bf Case 3}: Three critical points for $P(c)$ and four critical points for $P^d(\sigma)$, all belonging to $\cal{S}\rm^+_a$. The values of the parameters are  $y=1$, $x=4$, $w=2$, $\alpha=\frac{\sqrt{2}}{2}$ and $\beta=0.22$ (see Figure \ref{3b.22}).  This case is similar to the previous one, and the solution of the dual problem should be the critical point that corresponds to a minimum in the primal problem with the value of $\sigma$ closer to zero.

\begin{figure}[h]
\begin{center}
\includegraphics[scale=.42]{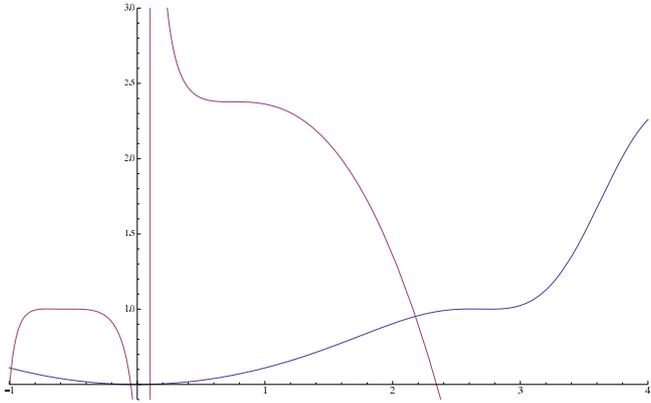}
\caption{Primal(in blue) and dual(in red) functions for the Case 4 with three critical points in the primal and two critical points in $\cal{S}\rm^+_a$ and two critical points in $\cal{S}\rm^-_a$ and $\sigma_f$ as a local maximum.}
\label{3nr}
\end{center}
\end{figure}

{\em\bf Case 4}: Three critical points in the primal and four critical points in the dual, but with two critical points in $\cal{S}\rm^+_a$, two critical points in $\cal{S}\rm^-_a$ and $\sigma_f$ as local maximum. The values of the parameters are $y=1$, $x=8$, $w=2$, $\alpha=\frac{\sqrt{2}}{2}$ and $\beta=0.25$ (see Figure \ref{3nr}).  If the value of the hyper parameter $\beta$ is reduced it is possible to make $\sigma_f$ into a local minimum and return in one of the previous cases.

{\em\bf Case 5}: One critical point in the primal problem and two critical points in the dual problem. This case occurs when the quadratic term with beta dominates the error function $W(x)$. If this case occurs, it means that the  value of $\beta$ is too big and the problem is not related with the original anymore, so one should choose a smaller value of $\beta$ to have a problem related to the original.

\noindent
Based on the study of these cases, we can obtain the general idea to find the best solution, i. e. the hyper parameter $\beta$ should be set to a value that satisfies condition (\ref{duddcon}) in order to have $\sigma_f$ as a local minimum, then search for the critical point in the domain $\cal{S}\rm^+_\sharp$.

\section{Conclusions}

In this paper we have presented an application of the canonical duality theory to function approximation using Radial Basis Functions. By using the sequential dual canonical transformation, the non convex problem with a general RBF function $\phi(\cdot)$ is reformulated in a canonical dual form. An associated strong duality theorem is also proposed.

\noindent
Applications to one of the most used RBF, the exponential function, are illustrated. Due to the particular properties of the exponential function, we are able to find a linear relation between the dual variables, which leads to an explicit form of the canonical dual problem. We also found conditions on the hyper parameter $\beta$ in order to obtain a reliable domain where to search for the best solution.
This research reveals an important phenomenon in complex systems, i.e. the global optimal solution may not be 
the best solution to the problem considered.  

\noindent
There are still several open topics on the application of the canonical duality theory to Radial Basis Error functions. For example there are other kinds of RBF that can be analyzed, like the multi quadratic  and the multi quadratic inverse functions, a further development for future research is to expand the one dimensional case to the multidimensional case with also  considering $w$ as a variable and not as a parameter. When this case is analyzed, we will be able to realize RBF neural networks based on canonical duality theory.

\end{document}